\documentclass{article}

\usepackage[preprint]{neurips_2026}

\usepackage[utf8]{inputenc} 
\usepackage[T1]{fontenc}    
\usepackage{hyperref}       
\usepackage{url}            
\usepackage{booktabs}       
\usepackage{amsfonts}       
\usepackage{nicefrac}       
\usepackage{microtype}      
\usepackage{xcolor}         

\usepackage{wanqiao_template}

\title{An Information-Theoretic Definition for \\ Open-Ended Learning}

\author{Wanqiao Xu$^*$  \\
    \texttt{wanqiaoxu@stanford.edu}
    \And
    Yifan Zhu$^*$ \\
   \texttt{zhuyifan@stanford.edu}
    \And
    Benjamin Van Roy \\
    \texttt{bvr@stanford.edu}
}

\def\BB#1{\mathbb{#1}}

\def\E{\BB{E}}

\def\P{\BB{P}}
\def\R{\BB{R}}

\def\one{\mathbf{1}}

\def\actions{\mathcal{A}}

\def\I{\mathbb{I}}
\def\H{\mathbb{H}}

\def\N{\mathbb{N}}
\def\P{\mathbb{P}}
\def\KL{\mathbf{d}_{\mathrm{KL}}}
\def\Re{\mathbb{R}}

\newcommand{\Reg}{\operatorname{Reg}}
\newcommand{\op}{\mathrm{op}}

\newcommand{\argmax}{\mathop{\mathrm{argmax}}}

\newcommand{\Tr}{\mathrm{Tr}}
\DeclareMathOperator{\Cov}{Cov}

\DeclareMathOperator{\rank}{rank}

\newcommand{\reward}[3]{r_{#1}\left(#2, #3\right)}
\newcommand{\rewardMean}[2]{\overline{r}_{#1}\left(#2\right)}
\newcommand{\bitEquiv}[1]{B_{#1}}
\newcommand{\bitEquivAvg}[2]{\overline{B}_{#1}\left(#2\right)}
\newcommand{\infoGain}[2]{\gamma_{#1}\left(#2\right)}
\newcommand{\openEndedBandit}{insatiable linear bandit}
\newcommand{\openEndedLogisticBandit}{insatiable logistic bandit}
\usepackage{amssymb}
\usepackage{pifont}
\newcommand{\cmark}{\ding{51}}%
\newcommand{\xmark}{\ding{55}}%
\newcommand{\markref}[2]{\makebox[1.2em][c]{#1}\makebox[5em][l]{\,#2}}

\begin{document}
\def\thefootnote{*}
\footnotetext{Equal contribution, ordered alphabetically.}
\maketitle

\maketitle

\begin{abstract}
A growing body of work points to the great promise of AI systems that can continually expand their capabilities as they operate in an open-ended environment.  But yet there is no coherent definition of open-endedness or theory about how an agent ought to explore an open-ended environment.  We introduce an information-theoretic definition based on a new concept -- the {\it bit-equivalent} -- which quantifies the information required to attain each level of expected reward.  
We consider an environment to be open-ended if an agent can attain linear growth in the bit-equivalent.  
We establish that classical bandit environments are not open-ended and formulate a bandit environment that is.   We also introduce an algorithm that achieves open-ended learning in this environment.  

\end{abstract}

\section{Introduction}
\label{sec:introduction}


A growing body of work points to the great promise of AI systems that can continually expand their capabilities as they operate in an open-ended environment.  Much of this literature characterizes open-ended environments as those that present increasingly novel and complex tasks \citep{Wang2019PairedOT,Wang2020EnhancedPOET,Earle2021VideoGamesTestbed,Zhang2023Omni,Faldor2025OmniEpic,Lu2025Automated}.  There is also work on the design and evaluation of agents for such environments \citep{Lehman2011AbandoningObjectives,Klissarov2024Motif,Robeyns2025Sica,Zheng2025Mcu,Romero2025HGRAILAR}.

Despite strong interest in the subject, there is no coherent definition of open-endedness or theory about how an agent ought to explore an open-ended environment.  Recent attempts to define open-endedness 
\citep{Sigaud2024DefinitionOpenEndedLearningProblems,Hughes2024PositionOpenEndedness} identify an environment as open-ended if the sequence of artifacts an agent produces is both novel and learnable, from the perspective of an observer.  In this context, {\it novelty} means that there is greater uncertainty about future relative to current observations, and {\it learnable} means that current observations inform predictions of future observations.

However, novelty and learnability do not fully capture open-endedness.  In particular, these measures do not indicate whether an agent can improve capabilities through interaction with the environment.  An agent that continually generates novel and learnable policies does not necessarily acquire information that enables improved performance.


In this paper, we propose a new information-theoretic definition of what makes an environment open-ended.  We focus on bandit environments to simplify our analysis and make our key insights transparent.  
Loosely speaking, we consider an environment open-ended if sustained reward improvement requires continual acquisition of useful information.  To formalize this intuition, we introduce a new concept: the {\it bit-equivalent}.  The bit-equivalent $B_\rho$ of expected reward $\rho$ is the minimum amount of information about the environment required to attain expected reward $\rho$.  Intuitively, the bit-equivalent captures not how much reward is attainable, but how much useful information is required to attain that.  
We call an environment open-ended if there exists an agent that generates rewards for which the bit-equivalent grows at a linear rate. 
And we say that an agent achieves open-ended learning in such an environment if it realizes this linear rate.

We establish that classical bandit environments, even those with unbounded rewards and infinite action sets, are not open-ended.  We then provide an example of an open-ended environment that is an infinite-dimensional linear-Gaussian bandit  
and establish that a variant of Thompson sampling achieves open-ended learning in this environment.  

\paragraph{Contributions.}
Our main contributions are as follows.
\begin{enumerate}
\item We introduce an information-theoretic definition of open-endedness for bandit environments based on the bit-equivalent of expected reward (\cref{sec:problem_formulation}).
\item We show that several classical bandit environments, including finite-armed, finite-dimensional linear, Gaussian process, and commonly studied infinite-armed bandits, are not open-ended (\cref{sec:non_openended}).
\item We formulate an infinite-dimensional linear-Gaussian bandit that is open-ended and introduce an algorithm that achieves open-ended learning in this environment (\cref{sec:linear_gaussian_bandit}). 
\end{enumerate}


\section{Problem Formulation}
\label{sec:problem_formulation}


We begin by introducing basic definitions and notation that we will use throughout this paper.  We will use standard big-$O$ notation to express asymptotic behavior of functions, including 
$\tilde{O}$, when we choose to ignore logarithmic factors.  
We use standard notations from the information theory literature and refer readers to \citet{Clover2006InfoTheory} for a detailed overview.  
Throughout this paper, we use $\log$ to denote the natural logarithm.  
All random variables are defined with respect to a common probability space $(\Omega, \mathcal{F}, \P)$.  

A bandit environment is defined by an action set $\actions$, an unknown parameter $\theta$, a reward function $r$, and a noise distribution.  We model $\theta$ as a random variable, representing uncertainty faced by the agent designer prior to interaction.
At each time $t\in\mathbb{N}_{\ge0}$, based on the history $H_t =  (A_0, R_1, A_1, R_2, \ldots, A_{t-1}, R_t)$, an agent $\pi$ 
selects an action $A_t \sim \pi(\cdot\mid H_t)$ and
receives a reward
\begin{align*}
  R_{t+1} = \reward{\theta}{A_t}{W_{t+1}},
\end{align*}
where $W_{t+1}$ is independent noise.
We denote the mean reward by
\[
\rewardMean{\theta}{a} = \E[ R_{t+1} \mid \theta, A_t = a ].
\]

It is common to characterize the performance of an agent in terms of the rate at which the reward or expected reward grows.  But that does not indicate whether an agent is achieving open-ended learning or whether the environment is open-ended.  For that, we need to study the amount of information required to attain levels of expected reward.  In some environments, large rewards may be obtained with little or no information, whereas even small rewards may require substantial information in others.  Thus, reward alone cannot distinguish between environments in which further improvement requires additional information and those in which it does not. The following definition formalizes the relationship between expected reward and the information required to achieve it.
We measure all information-theoretic quantities in nats, but use the term ``bit-equivalent'' where ``bit'' is used in the conventional sense to refer to information.  
\begin{definition}[Bit-equivalent]
\label{def:bit_equiv_of_expected_reward}
The bit-equivalent of expected reward level $\rho \in \R$ is
\[\bitEquiv{\rho} = \inf_{A: \E\left[\rewardMean{\theta}{A} \right] \geq \rho} \I(\theta; A),\]
where $\I(\theta; A)$ is the mutual information between $\theta$ and $A$ in nats, and the infimum is taken over all random variables $A$ such that $\E\left[ \rewardMean{\theta}{A} \right] \geq \rho$.
\end{definition}
In other words, $\bitEquiv{\rho}$ is the amount information about $\theta$ required, on average, to select an action that delivers expected reward at least $\rho$.

We characterize the performance of an agent $\pi$ in terms of the average bit-equivalent of the expected reward attained by $\pi$. 
\begin{definition}[Average bit-equivalent]
The average bit-equivalent of reward attained by an agent $\pi$ up to time $T$ is defined as
\begin{align*}
\bitEquivAvg{T}{\pi}
  = \frac{1}{T} \sum_{t=0}^{T-1} \bitEquiv{\E_\pi[R_{t+1}]},
\end{align*}
where $\E_\pi$ denotes the expectation over the randomness in $\theta$, actions generated by the policy $\pi$, and the noise.
\end{definition}

We are interested in environments in which useful information can be acquired at a nonvanishing asymptotic rate. 
We formalize this via linear growth rate of $\bitEquivAvg{T}{\pi}$.
\begin{definition}[Open-endedness]
\label{def:open-ended}
An environment is open-ended if there exists an agent $\pi$ such that $\bitEquivAvg{T}{\pi} = \Omega(T)$.
\end{definition}
We say that an agent achieves open-ended learning in such an environment if it realizes this linear rate.

\section{Classical Bandit Environments Are Not Open-Ended}
\label{sec:non_openended}

Having defined open-endedness through the growth rate of the average bit-equivalent, we now ask which bandit environments satisfy this property.
We show that most classical bandit environments are not open-ended: either information about the environment cannot be acquired at a linear rate, or the information that can be acquired does not translate into sustained reward improvement.


aWe begin with a standard notion from the bandit literature.
The information gain of an agent $\pi$ up to time $T$ is
\[
\infoGain{T}{\pi} = \I(\theta; H_T),
\]
which measures the total information about $\theta$ contained in the history.
This quantity upper bounds the average bit-equivalent.

 
\begin{lemma}
  \label{lem:info_gain_upper_bound}
  For all agents $\pi$ and all $T \ge 1$, $\bitEquivAvg{T}{\pi} \le \infoGain{T}{\pi}$.
\end{lemma}
\begin{proof}
  Applying the data processing inequality on the Markov chain $\theta \to H_t \to A_t$ gives
  \begin{align*}
    \bitEquiv{\E_\pi[R_{t+1}]} = \inf_{A: \E\left[\rewardMean{\theta}{A} \right] \geq \E_\pi[R_{t+1}]} \I(\theta; A) 
    \le \I(\theta; A_t) 
    \le \I(\theta; H_t).
  \end{align*}
  The result follows by averaging over $t$, and noting that $\I(\theta; H_t) \le \I(\theta; H_T)$ for all $t \le T$.
\end{proof}

Consequently, sublinear information gain for every agent implies non-open-endedness. This criterion rules out finite-armed (\Cref{thm:mab_info_gain}) and finite-dimensional linear bandits with Gaussian noise (\Cref{thm:linear_bandit_info_gain}), as well as Gaussian process bandits under standard kernel assumptions (\Cref{thm:gp_bandit_info_gain}). However, sublinear information gain is sufficient but not necessary: finite-armed bandits with general noise (\Cref{thm:mab_bit_equiv_bound}) and infinite-armed bandits with i.i.d. means (\Cref{thm:mab_iid}) can have linear or even infinite information gain, yet remain non-open-ended because the acquired information does not support sustained growth in reward-relevant bit-equivalent.

\Cref{tab:non_openended_bandits} summarizes these cases.

\begin{table}[ht]
\caption{Classic environments are not open-ended. Sublinear maximal information gain is sufficient but not necessary for non-open-endedness.}
\label{tab:non_openended_bandits}
\centering
\begin{tabular}{lcc}
\toprule
Environment & Linear info gain? & Open-ended? \\
\midrule
Finite-armed bandit with Gaussian noise & \markref{\xmark}{[\Cref{thm:mab_info_gain}]}& \markref{\xmark}{}\\
Finite-armed bandit with non-Gaussian noise & \markref{\cmark}{} & \markref{\xmark}{[\Cref{thm:mab_bit_equiv_bound}]}\\
$d$-dimensional linear bandit with Gaussian noise & \markref{\xmark}{[\Cref{thm:linear_bandit_info_gain}]} & \markref{\xmark}{} \\
Gaussian process bandit & \markref{\xmark}{[\Cref{thm:gp_bandit_info_gain}]} & \markref{\xmark}{} \\
Infinite-armed bandit with i.i.d. means & \markref{\cmark}{} & \markref{\xmark}{[\Cref{thm:mab_iid}]}\\
\bottomrule
\end{tabular}
\end{table}
The remainder of this section establishes asymptotic bounds that support these conclusions.

\subsection{Bandit environments with sublinear information gain}
Throughout this subsection, we assume that the noise is additive Gaussian, i.e., $\reward{\theta}{A_t}{W_{t+1}} = \rewardMean{\theta}{A_t} + W_{t+1}$ and $W_{t+1} \sim \mathcal{N}(0, \sigma^2)$ for some $\sigma > 0$.  This allows us to show that the information gain $\infoGain{T}{\pi}$ grows sublinearly in $T$, implying non-open-endedness.  

\paragraph{Finite-dimensional linear bandits}
A $d$-dimensional linear bandit is characterized by $d$-dimensional $\theta$ and $\actions$, and $\rewardMean{\theta}{a} = \theta^T a$.

\begin{restatable}[Information gain bound for linear bandits]{theorem}{infoGainBoundForLinearBandits}
\label{thm:linear_bandit_info_gain}
Consider a $d$-dimensional linear bandit with Gaussian noise.
Let the action set $\actions \subseteq \{a \in \mathbb{R}^d : \|a\|_2 \le 1\}$.
Let $\theta \in \mathbb{R}^d$ be a random vector with covariance matrix $\Sigma$. 
Then, for all $T \ge 1$ and all agents $\pi$,
\begin{align*}
\infoGain{T}{\pi}
\le
\frac{1}{2} \log \det\!\left(
I_d + \frac{T}{\sigma^2}\Sigma
\right) + \frac{1}{2}
= O(d \log T)
= o(T)
.
\end{align*}
\end{restatable}
The proof of this and other results of this section are deferred to Appendix~\ref{app:proofs_for_non_opendended}.


\paragraph{Finite-armed bandits}

Finite-armed bandits are bandit environments with $|\actions| < \infty$.
Formally, $\theta\in\R^{|\actions|}$ and $\rewardMean{\theta}{a} = \theta_a$ for all $a\in\actions$. 
A finite-armed bandit can be represented as a $|\actions|$-dimensional linear bandit with the action set $\actions$ being standard basis vectors.  With this representation, \cref{thm:linear_bandit_info_gain} can be applied directly as follows.  
\begin{theorem}[Information gain bound for finite-armed bandits]
\label{thm:mab_info_gain}
    Consider a finite-armed bandit with Gaussian noise $W_{t+1} \sim \mathcal{N}(0, \sigma^2)$ and let 
    $\theta = (\theta_a)_{a\in\actions} \in \R^{|\actions|}$ be a random vector with finite covariance matrix $\Sigma$.
    Then for all $T\ge 1$ and all agents $\pi$, 
    \[\infoGain{T}{\pi} \le \frac{1}{2} \log \det \left( I_{|\actions|} + \frac{T }{\sigma^2}\Sigma \right) + \frac{1}{2} = O(|\actions|\log T)=o(T).\]
\end{theorem}

\paragraph{Gaussian process bandits}

Gaussian process bandits are the natural infinite-dimensional analogue of finite-dimensional linear bandits and are a standard model in kernelized bandit optimization \citep{Chowdhury2017OnKernelizedMAB,Vakili2021InfoGainGP}. In the
typical Gaussian process bandit, the action set $\actions\subset\R^d$ is compact,
the mean reward function $\rewardMean{\theta}{a} = \theta(a)$, and
$\theta \sim \mathrm{GP}(0,k)$ is sampled from a Gaussian process with kernel
$k$.

Fix a finite reference measure $\mu$ on $\actions$. Under standard kernel
conditions used in the Gaussian process bandit literature
\citep{Chowdhury2017OnKernelizedMAB,Vakili2021InfoGainGP}, the kernel admits a
Mercer decomposition with respect to $\mu$.  Concretely, there exists
$\{(\lambda_m,\phi_m)\}_{m \ge 1}$ with eigenvalues
$\lambda_1 \ge \lambda_2 \ge \cdots \ge 0$ such that the functions
$\{\phi_m\}_{m\ge 1}$ are orthonormal in $L^2(\mu)$ and
\begin{align*}
    k(a,a') = \sum_{m=1}^{\infty} \lambda_m \phi_m(a)\phi_m(a').
\end{align*}
Further, assume that $|k(a,a')| \le 1$ for all $a,a'\in\actions$ and
$|\phi_m(a)| \le 1$ for all $m\ge 1$ and $a\in\actions$. The eigenvalues are
summable under these assumptions:
\begin{align*}
    \sum_{m=1}^{\infty}\lambda_m
    =
    \int_{\actions} k(a,a)\,d\mu(a)
    \le \mu(\actions)
    < \infty.
\end{align*}
Further, the spectral tail, defined by
\(
    \delta_D = \sum_{m=D+1}^{\infty} \lambda_m,
\)
satisfies $\delta_D \to 0$ as $D\to\infty$.
This spectral tail decay rules out linear information gain, as we establish in the following theorem, which is adapted from Theorem~3 of
\citet{Vakili2021InfoGainGP}.
\begin{restatable}[Information gain bound for Gaussian process bandits]{theorem}{infoGainBoundForGPBandits}
\label{thm:gp_bandit_info_gain}
Consider the Gaussian process bandit defined above with Gaussian noise of
variance $\sigma^2$. Then, for all $T\ge 1$, all agents $\pi$ and all $D \ge 1$,
\begin{align}
    \infoGain{T}{\pi}
    \le
    \frac{1}{2}D\log\left(1+\frac{T}{\sigma^2D}\right)
    +
    \frac{\delta_D T}{2\sigma^2}. \label{eqn:gaussian_info_gain_bound}
\end{align}
Consequently, if $\delta_D \to 0$ as $D\to\infty$, then
\begin{align}
    \sup_{\pi}\infoGain{T}{\pi} = o(T). \label{eqn:information_gain_bound_GP}
\end{align}
\end{restatable}

To see the final implication, fix any $D\in\N$. The information gain bound \eqref{eqn:gaussian_info_gain_bound} implies:
\begin{align*}
    \frac{\sup_{\pi}\infoGain{T}{\pi}}{T}
    \le
    \frac{D}{2T}\log\left(1+\frac{T}{\sigma^2D}\right)
    +
    \frac{\delta_D}{2\sigma^2}.
\end{align*}
Taking $\limsup_{T\to\infty}$ yields
\begin{align*}
    \limsup_{T\to\infty}\frac{\sup_{\pi}\infoGain{T}{\pi}}{T}
    \le
    \frac{\delta_D}{2\sigma^2}.
\end{align*}
Since $D$ is arbitrary and $\delta_D\to 0$, \eqref{eqn:information_gain_bound_GP} follows.  

As shown, Gaussian process bandits satisfying the above spectral conditions do not permit linear information gain and are non-open-ended under \cref{def:open-ended}. In
particular, Gaussian process bandits with polynomial eigenvalue decay
$\lambda_m = O(m^{-\alpha})$ for some $\alpha>1$, or with exponential
eigenvalue decay, are non-open-ended. These decay conditions cover the common Matérn and squared exponential kernels analyzed in the GP bandit literature \citep{Chowdhury2017OnKernelizedMAB,Vakili2021InfoGainGP}.

\subsection{Bandit environments where linear information gain does not imply open-endedness}

While we have established that sublinear information gain implies non-open-endedness, the reverse does not hold in general:  linear information gain alone does not imply open-endedness. Some environments allow agents to acquire information at a linear rate, yet remain non-open-ended because this information does not support sustained growth in the bit-equivalent.
Finite-armed bandits with general noise and infinite-armed bandits fall into this category.

\paragraph{Finite-armed bandits with non-Gaussian noise}

We showed in \cref{thm:mab_info_gain} that the information gain of finite-armed bandits with Gaussian noise is bounded sublinearly, implying non-open-endedness.  Yet this sublinear bound may not be satisfied when the noise is no longer assumed to be Gaussian.  Consider the following example.  Suppose $(\theta_a)_{a\in\actions}$ are i.i.d.
$\mathrm{Unif}([0,1])$ and observations are noiseless:
\(
R_{t+1} = \theta_{A_t}
\).
If the agent selects a fixed action $a$ at time $0$, then
$H_1=(a,R_1)=(a,\theta_a)$ reveals the exact value of the continuously
distributed parameter $\theta_a$.  Thus $\I(\theta;H_1)=\infty$.  
However, this does not imply open-endedness.  

As we will show, the
finite action set itself imposes a stronger limitation: any action can be
specified using at most $\log |\actions|$ nats of information.  It follows that
the bit-equivalent of any attainable reward level is uniformly bounded,
regardless of the observation noise.

\begin{theorem}
\label{thm:mab_bit_equiv_bound}
    Let $|\actions|<\infty$.  Then for all $T\ge 1$ and all agents $\pi$,
    $\bitEquivAvg{T}{\pi} \le \log|\actions|$.
    In particular, the environment is non-open-ended.  
\end{theorem}
\begin{proof}
    For all random variables $A$ taking values in a finite set $\actions$, 
    $\I(\theta;A) \le \H(A) \le \log|\actions|$.    Hence, $B_\rho$ is uniformly bounded above by $\log|\actions|$ for every attainable reward level $\rho$.  The result follows by averaging over $t=0,1,\ldots,T-1$.  
\end{proof}
This bound highlights a limitation of information gain as a proxy for
open-endedness.  The information gain $\infoGain{T}{\pi}$ measures all
information acquired about the environment, regardless of whether that information 
leads to high reward.  By contrast, the bit-equivalent measures only
the information required to select an action achieving a given reward level.
Indeed, without assumptions on the observation noise, information gain can be
infinite even in a finite-armed bandit, while the average bit-equivalent remains
bounded.  

\paragraph{Infinite-armed bandits}

Prior works studying bandits with infinitely many arms \citep{Berry1997InfiniteMAB,Wang2008AlgorithmsInfiniteMAB,Carpentier2015SimpleRegretInfiniteMAB} commonly assume that $\actions$ is countably infinite and that all arms are i.i.d. before making any observations.  We show that infinite-armed bandits under these assumptions are non-open-ended, even though they may admit unbounded information gain $\infoGain{T}{\pi}$ for some agents $\pi$. 
\begin{restatable}[Bit-equivalent bound for infinite-armed bandit]{theorem}{bitEquivBoundInfiniteMAB}
\label{thm:mab_iid}
    Consider an infinite-armed bandit with i.i.d. arm means and independent noise.  Then for all $T\ge 1$ and all agents $\pi$, 
    \[\bitEquivAvg{T}{\pi} \le \log T.\]
    In particular, the environment is non-open-ended.  
\end{restatable}
We posit that a qualitatively similar conclusion holds even without the i.i.d. arm mean assumption and leave the analysis to future work. 

Our results on infinite-armed bandits suggest that countably infinite actions alone does not imply open-endedness.  
For the commonly studied infinite-armed bandits with i.i.d. means, the average bit-equivalent achieved by any policy grows at most logarithmically.  
We hypothesize that independent arm means is sufficient for non-open-endedness for infinite-armed bandits.  Intuitively, the weak coupling across arms prevents linear growth of useful information.   

\section{An Open-Ended Environment}
\label{sec:linear_gaussian_bandit}

In this section, we present a concrete example of an open-ended environment.
The environment is an infinite-dimensional linear-Gaussian bandit.  It can also
be viewed as a Gaussian process bandit, but one that falls outside the standard
kernel conditions considered in the preceding section.  To attach a name, we will call our example the {\it \openEndedBandit}.

Let
\(
\actions
=
\left\{
a\in\{0,1\}^{\mathbb{N}} : \|a\|_1 < \infty
\right\}
\).
Thus, each action is an infinite-dimensional binary vector with finite support.
Let $\theta\in\mathbb{R}^{\mathbb{N}}$ have independent components
\(
\theta_i \overset{\mathrm{iid}}{\sim} \mathcal{N}(-1,1)
\).
The mean reward of action $a\in\actions$ is
\(
\rewardMean{\theta}{a} = \langle \theta,a\rangle
\).
Observed rewards are perturbed by additive Gaussian noise:
\[
R_{t+1}
=
r_\theta(A_t,W_{t+1})
=
\rewardMean{\theta}{A_t} + W_{t+1}
=
\langle \theta,A_t\rangle + W_{t+1},
\qquad
W_{t+1}\sim\mathcal{N}(0,\sigma^2),
\]
where $\sigma>0$.
To make sure the reward is well defined, we restrict to agents such that $\E [ \|A_t \|_0 ] < \infty$ for all $t$.

The \openEndedBandit{} differs in an important way from the Gaussian process bandits addressed by \cref{thm:gp_bandit_info_gain}.  In the coordinate basis, the kernel eigenvalues are not summable.  This violates the spectral tail condition that rules out linear information gain.


The next theorem establishes that the \openEndedBandit{} is open-ended. 
It directly follows from \cref{thm:TTS_success}, to be discussed in later sections.
We defer proofs of results in this section to Appendix~\ref{app:proofs_for_openended}. 
\begin{restatable}[]{theorem}{LinearBanditOpenEnded}
\label{prop:linear_bandit_open_ended}
The \openEndedBandit{} is open-ended.  
\end{restatable}

A careful reader may notice that rewards generated by the \openEndedBandit{} are unbounded and may wonder if that is required for open-endedness.  A simple variant -- the \openEndedLogisticBandit -- proves otherwise.  Suppose rewards are binary with mean
\(\rewardMean{\theta}{A_t} = g(\langle \theta, A_t\rangle)\),
where $g(x) = (1 + e^{-x})^{-1}$ is the logistic function.  Then, the resulting bandit environment remains open-ended (see \cref{thm:generalized_linear_bandit_openended} in Appendix~\ref{app:proofs_for_openended}).

\subsection{Classical bandit algorithms fail to achieve open-ended learning}

Classical bandit algorithms balance between exploration and exploitation to maximize cumulative reward \citep{Russo2014TS,Auer2002FiniteTimeAnalysis}.  However, most existing bandit algorithms are designed to operate in classical bandit environments.  Many of these environments, as we have shown in \cref{sec:non_openended}, are non-open-ended.

In the \openEndedBandit{}, these algorithms fail to achieve open-ended learning.  There are two failure cases.  Algorithms that do not constrain the action set, such as Thompson sampling, attempt to explore too many coordinates.  In particular, TS does not produce valid actions, because na\"ive optimization over the countably infinite action set does not yield a finite maximizer.  The step-wise expected reward attained by a TS agent diverges to $-\infty$ as the optimization set expands.  

Algorithms that truncate the action set to a fixed finite dimension limit the amount of information that can be extracted.  We show that the average bit-equivalent of such an algorithm is uniformly bounded.

\paragraph{Thompson Sampling (TS)} The TS agent $\pi_{\mathrm{TS}}$ keeps an infinite-dimensional posterior over $\theta$.  At round $t$, the TS agent samples a parameter $\theta^{(t)} \sim \P(\theta\in\cdot \mid H_t)$ from the posterior, then takes an action $a$ that maximizes $a^\top\theta^{(t)}$.  
\begin{restatable}[Thompson sampling failure]{theorem}{TSFailure}
\label{thm:TS_failure}
    In the \openEndedBandit{}, for all $t \ge 1$ and history $H_t$ generated by an valid policy, 
    \[\P\left( \sup_{a\in\actions} a^\top \theta^{(t)} = \infty \mid H_t \right)  = 1,\]
    where $\theta^{(t)}\sim \P(\theta\in\cdot \mid H_t)$.  
    In particular, the supremum is not attained by any action in $\actions$, and $\pi_{\mathrm{TS}}$ does not produce valid actions.  
\end{restatable}
This failure shows that direct application of TS results in an ill-defined policy.  The optimization objective for the posterior sample $\theta^{(t)}$ has infinite supremum over $\actions$, and no action in $\actions$ attains it.  One might hope that this is a mere technical issue due to optimization over a non-compact set, mitigated by restricting to finite truncations $\actions_M$ and letting $M\to\infty$.  
The next theorem shows that this is not the case.  Although each truncated optimization problem has a well-defined maximizer, the expected reward of the maximizer diverges to $-\infty$ as the truncation approaches the full action set $\actions$.  
\begin{restatable}[Reward attained by Thompson sampling tends to negative infinity]{theorem}{TSFailureReward}
\label{thm:TS_failure_reward}
    In the \openEndedBandit{}, for all $M\in\mathbb{N}$ and $t \ge 1$, let 
    \[\actions_M = \{a\in\actions : a_i = 0 \text{ for all } i > M\}\]
    denote the finite-dimensional truncation of $\actions$, and let 
    \[A_{t,M} \in \argmax_{a\in\actions_M} a^\top \theta^{(t)},\]
    where $\theta^{(t)} \sim \P(\theta\in\cdot \mid H_t)$.  
    Then, for all history $H_t$ generated by an valid policy, 
    \[\P\left(\lim_{M\to\infty} \E\left[ A_{t,M}^\top \theta \mid H_t\right] = -\infty\right) = 1.\]
\end{restatable}


\paragraph{Fixed Truncation (FT)} An FT agent $\pi^M$ fixes a truncation window $M \in\mathbb{N}$ and only selects actions supported on the first $M$ coordinates:
\[\actions_M = \{a\in\actions \mid a_i  = 0 \text{ for all } i > M\}.\]
It may run any bandit algorithm for finite-dimensional linear bandit on this constrained action set.  
\begin{restatable}[FT failure]{theorem}{FTFailure}
\label{thm:FT_failure}
    In the \openEndedBandit{}, for all $M\in\mathbb{N}$, any FT agent $\pi^M$ satisfies
    \[\bitEquivAvg{T}{\pi^M} \le M\log 2 \qquad \text{for all } T\ge 1.\]
    In particular, no fixed-$M$ FT agent achieves open-ended learning.  
\end{restatable}

\subsection{Thompson sampling with a sequence of learning targets}
The preceding failure cases of classical bandit algorithms show that directly learning the full parameter $\theta$ leads to aggressive exploration strategies, resulting in negative infinite instantaneous rewards.  A fixed truncation avoids this pitfall but prevents the agent from gaining a linear amount of useful information.  

This motivates an intermediate design principle.  Instead of learning the full
environment parameter at once, or fixing a single finite-dimensional
approximation, the agent should pursue a sequence of learning targets of
increasing complexity.  Each target captures a finite but growing portion of the
environment.  The agent first learns a simple target that supports modest
reward, and then gradually moves to richer targets that support higher reward.

This idea is related to satisficing Thompson sampling
\citep{russo2022satisficing}, which
modifies Thompson sampling by probability matching to an alternative learning
target, such as a satisficing action, rather than to the optimal action.  It is
also related to subsequent work on designing learning targets through a
rate-distortion tradeoff \citep{Arumugam2021BLASTS}.  Our use of
learning targets is different in emphasis: the target is not chosen to trade off
a fixed information cost against a fixed approximation error, but rather to
increase in complexity over time so that the agent can sustain linear growth in
reward-relevant information.

We now formalize this idea.  A \emph{learning target} is a random variable
$\chi$ that represents partial information about the environment parameter
$\theta$.  Given a learning target $\chi$, define an optimal action with respect to the information contained in $\chi$ by
\begin{align*}
    a^*(\chi) \in \argmax_{a\in\actions} \E\left[\theta^\top a \mid \chi\right].
\end{align*}
Thus $a^*(\chi)$ is the best action an agent could choose if it knew $\chi$.

Ordinary Thompson sampling corresponds to the special case $\chi=\theta$: at
each time, the agent samples a full parameter from the posterior and acts
optimally for that sampled parameter.  More generally, Thompson sampling can be
applied to any learning target $\chi$.  At time $t$, given history $H_t$, the
agent samples
\(
    \chi_t \sim \P(\chi\in\cdot \mid H_t)
\),
and then selects an action
\(
    A_t \in a^*(\chi_t)
\).
We refer to this procedure as Thompson sampling with learning target $\chi$.

For open-ended learning, we consider a sequence of learning targets
\(
    \chi^1,\chi^2,\dots
\),
where later targets encode increasingly rich information about the environment.
At time $t$, the agent chooses an index $m_t$ and applies Thompson sampling with
respect to the target $\chi^{m_t}$.  That is, it samples
\begin{align*}
    \chi_t^{m_t}
    \sim
    \P(\chi^{m_t}\in\cdot \mid H_t),
\end{align*}
and then chooses
\(
    A_t \in a^*(\chi_t^{m_t})
\).
The role of the schedule $m_t$ is to balance two competing goals: the target
$\chi^{m_t}$ should be rich enough to support high reward, but simple enough
that it can be learned effectively from data collected so far.

\subsection{Truncated Thompson sampling achieves open-ended learning}
\label{sec:TTS_openended}

In the \openEndedBandit{}, the natural learning targets
are finite-coordinate truncations.  Let
\(
    \chi^m = \theta_{1:m}
\).
Knowing $\chi^m$ reveals the first $m$ coordinates of the parameter and ignores
the rest.  Thompson sampling with target $\chi^m$ is therefore equivalent to
running Thompson sampling on the truncated action set
\begin{align*}
    \actions_m = \left\{ a\in\actions : a_i=0 \text{ for all } i>m \right\}.
\end{align*}

Truncated Thompson sampling (TTS) applies this idea with a truncation level that
may grow over time.  At round $t$, given history $H_t$, the TTS agent selects a
truncation level $m_t=m_t(H_t)$, samples
\begin{align*}
    \theta^{(t)}_{1:m_t}
    \sim
    \P\!\left(\theta_{1:m_t}\in\cdot \mid H_t\right),
\end{align*}
and chooses
\begin{align*}
    A_t
    \in
    \argmax_{a\in\actions_{m_t}}
    a_{1:m_t}^{\top}\theta^{(t)}_{1:m_t}.
\end{align*}
Thus TTS avoids both the invalidity of full Thompson sampling and the bounded performance of any fixed truncation.

\begin{restatable}[TTS success]{theorem}{TTSSuccess}
\label{thm:TTS_success}
    In the \openEndedBandit{}, there exists a truncation
    schedule $(m_t)_{t\ge 0}$ such that the resulting TTS agent
    $\pi_{\mathrm{TTS}}$ satisfies, for all $T\ge 1$,
    \[
        \bitEquivAvg{T}{\pi_{\mathrm{TTS}}}
        =
        \Omega(T).
    \]
    In particular, TTS achieves open-ended learning.
\end{restatable}

In fact, this average bit-equivalent rate achieved by TTS is optimal, in the sense that there exists a matching upper bound.  
\begin{restatable}[]{theorem}{MatchingUpperBound}
\label{thm:matching_upper_bound}
	In the \openEndedBandit{}, for all $T\ge 1$ and all agents $\pi$, 
	\[
	\bitEquivAvg{T}{\pi} = O(T).
	\]
\end{restatable}

Our results on \openEndedBandit{} shows that designing algorithms for open-ended environments involves carefully thinking about the learning target.  We view TTS as a sanity check that intuition about open-ended learning -- that an agent ought to continually expand its learning target -- yields an algorithm capable of open-ended learning. Alternative designs such as GP-UCB-type agents \citep{Srinivas2010GPUCB} with carefully tuned confidence widths may also achieve open-ended learning, though we leave their analysis to future work.  While these algorithms are promising, their need to prescribe sequences of learning targets or confidence set widths may render them fragile in practice.  
We discuss further limitations in the next section.

\section{Discussion}
\label{sec:discussion}

Two key aspects of open-endedness studied in this work are the definition of open-ended environments and the design of open-ended learning agents.  In this section, we begin with an overview of prior work, which largely focuses on empirical designs of open-ended agents and environments, as well as a small handful of recent papers attempting to formalize open-endedness.  We conclude with a discussion of limitations and future directions.  

\subsection{Related work}

\paragraph{Open-endedness in artificial life and AI.}
Open-endedness has a long history that is rooted in artificial evolution and leads up to modern AI. Early work in artificial life views open-endedness as the ability of evolutionary systems to continually increase novelty and complexity \citep{Standish2003OpenendedArtificialEvolution,Lehman2011AbandoningObjectives}. More recent work has revisited the definition and evaluation of open-endedness  \citep{Stepney2024OpenEndednessOfDetect,Soros2024Creativity}. With the development of modern deep learning systems, this perspective of open-endedness leads to methods that generate novel environments such as POET and its variants \citep{Wang2019PairedOT,Wang2020EnhancedPOET}, testbeds for open-ended phenomena \citep{Earle2021VideoGamesTestbed}, and methods that use models of interestingness to generate tasks or environments \citep{Zhang2023Omni,Faldor2025OmniEpic}. 
Recent years have seen incredible advancements in modern AI systems.  Naturally, interests arise in self-improving agents \citep{Schaul2024BoundlessSocraticLearning,Lu2025Automated,Robeyns2025Sica,Zhang2026Darwin}, autonomous scientific discovery \citep{Lu2026AIScientist}, and open-ended evaluation \citep{Zheng2025Mcu}.  

\paragraph{Formal definitions of open-endedness.}
Despite a growing body of empirical work, there remains no rigorous definition of open-endedness.  
Formal accounts by \citet{Sigaud2024DefinitionOpenEndedLearningProblems} and \citet{Hughes2024PositionOpenEndedness} characterize open-endedness in terms of novelty and learnability of artifacts from the perspective of an observer. Our work takes a step further in this line of research, offering a quantitative definition of open-endedness measured by whether increasingly valuable performance requires acquiring information at a nonvanishing rate. 

\paragraph{Reinforcement learning and continual learning.}
From an agent design perspective, our work also relates to reinforcement learning (RL) and continual learning. Intrinsic motivation and curiosity-based methods encourage agents to explore by rewarding heuristics of uncertainty and novelty \citep{Meyer1991CuriosityBoredom,Schmidhuber2010FormalTheory,Deepak2017CuriosityDriven,Burda2019ExplorationRandomDistill,Shyam2019ModelBasedExploration,Raileanu2020RIDE,Colas2020AutotelicAW,Du2023IntrinsicallyMotivated,Klissarov2024Motif}.  Hierarchical RL approaches the problem of open-ended learning via temporal structure discoveries \citep{Klissarov2025DiscoveringTemporalStructureOverview}.  Continual learning studies agents that learn over indefinite interaction and may face changing or expanding tasks \citep{Kirkpatrick2017Overcoming,Abel2023DefinitionCRL,Lewandowski2025TheWorldIsBigger}. 

\paragraph{Classical bandits and unbounded rewards.}
Finally, our technical analysis builds on and extends results in classical bandit models and algorithms.  A closely related prior work \citep{arumugam2024exploration} studies a relaxed instance of classical bandit models that accommodates unbounded rewards, a setting that also appeared in early work on dynamic programming and Markov decision processes \citep{Harrison1972DiscreteDPUnbounded,Lippman1973SemiMDPUnbounded,Lippman1975DPUnbounded,Nunen1978DPUnbounded,Hu1992DiscountedMDPUnbounded}.  While unbounded rewards and infinite action sets may appear to provide natural sources of open-endedness, we show that neither is sufficient.  Rather, open-endedness requires that useful information about the environment can be continually acquired to sustain performance improvement.

\subsection{Conclusion and limitations}
\label{sec:conclusion_and_limitations}

Our results provide a quantitative characterization of open-endedness in interactive learning, made precise through the notion of the bit-equivalent.  This framework 
provides grounds from which future theoretical analyses and agent-design principles may take inspiration.  

We discuss two main limitations and directions for future work.  First, our definition for open-endedness applies only to the bandit setting.  Extending this definition to accommodate stateful and nonstationary environments is an important next step.  Second, as discussed at the end of \cref{sec:TTS_openended}, the TTS agent design requires as input a sequence of learning targets.  Designing agents without the need for deliberate hyperparameter choices is an exciting future direction.

\bibliographystyle{plainnat}
\bibliography{references}


\newpage
\appendix

\section{Proofs for \cref{sec:non_openended}}
\label{app:proofs_for_non_opendended}

\subsection{Finite-dimensional linear bandits}

\infoGainBoundForLinearBandits*
\begin{proof}
Let $W \sim \mathcal{N}(0, \tau^2 I_d)$ be independent of $\theta$, and define
\[
U = \theta + W.
\]
Then by the chain rule of mutual information,
\begin{align*}
\I(\theta; H_T) \le \I(\theta; U, H_T) = \I(\theta; U) + \I(\theta; H_T \mid U).
\end{align*}

First we bound $\I(\theta; U)$.
We have
\begin{align*}
\I(\theta; U) = \I(\theta; \theta + W)
= h(\theta + W) - h(W).
\end{align*}
Using the Gaussian maximum-entropy bound,
\[
\I(\theta; U)
\le
\frac{1}{2} \log \det\!\left(I_d + \frac{1}{\tau^2}\Sigma\right).
\]
Now we bound $\I(\theta; H_T \mid U)$.
By the chain rule of mutual information,
\begin{align*}
\I(\theta; H_T \mid U)
=
\sum_{t=0}^{T-1} \left(
\I(\theta; R_{t+1} \mid H_t, U, A_t)
  + \I(\theta; A_t \mid H_t, U)
\right).
\end{align*}
Since $A_t$ is $H_t$-measurable, $\I(\theta; A_t \mid H_t, U) = 0$.
Thus we have
\begin{align*}
  \I(\theta; H_T \mid U)
  =
    \sum_{t=0}^{T-1} \I(\theta; R_{t+1} \mid H_t, U, A_t).
\end{align*}
Since $U=\theta+W$, we can write
\begin{align*}
R_{t+1}
=
A_t^\top U + (W_{t+1} - A_t^\top W).
\end{align*}
Thus
\begin{align*}
  \I (\theta; R_{t+1} \mid H_t, U, A_t)
  &= \I(\theta; W_{t+1} - A_t^\top W \mid H_t, U, A_t)
  \\
  &= h (W_{t+1} - A_t^\top W \mid H_t, U, A_t) - h(W_{t+1} - A_t^\top W \mid H_t, U, A_t, \theta)
  \\
  &= h (W_{t+1} - A_t^\top W \mid H_t, U, A_t) - h(W_{t+1} \mid H_t, U, A_t, \theta).
\end{align*}
Since Gaussian noise has the maximum entropy among all noise distributions with the same variance, we have
\begin{align*}
\I (\theta; R_{t+1} \mid H_t, U, A_t)  
  \le \frac{1}{2} \log \left(1 + \frac{\tau^2 \|A_t\|_2^2}{\sigma^2}\right)
  \le \frac{1}{2} \log \left(1 + \frac{\tau^2}{\sigma^2}\right).
\end{align*}

Summing over $t$, we have
\begin{align*}
\I(\theta; H_T)
  \le \frac{1}{2} \log \det\!\left(I_d + \frac{1}{\tau^2}\Sigma\right)
    + \frac{T}{2} \log \left(1 + \frac{\tau^2}{\sigma^2}\right).
\end{align*}
Now we take $\tau^2 = \frac{\sigma^2}{T}$ to obtain
\begin{align*}
\I(\theta; H_T)
  \le \frac{1}{2} \log \det\!\left(I_d + \frac{T}{\sigma^2}\Sigma\right) + \frac{T}{2} \log \left(1 + \frac{1}{T}\right).
\end{align*}
The result follows since $(1+1/T)^T \le e$.
\end{proof}

\subsection{Multi-armed bandits with infinitely many arms}

\bitEquivBoundInfiniteMAB*
\begin{proof}
    Fix $T\ge 1$ and an agent $\pi$.  Since the arm means $(\theta_a)_{a\in\actions}$ are i.i.d., the sequence is exchangeable.  Without loss of generality, we may relabel arms in the order they are first sampled by $\pi$.  Under this relabeling, $A_0,A_1,\ldots,A_{T-1} \subseteq \{1,2,\ldots,T\}$.  Define
    \[A_T^* \in \argmax_{a\in\{1,\dots,T\}} \rewardMean{\theta}{a},\]
    breaking ties arbitrarily.  Since $A_T^*$ takes values in a set of size $T$, $\I(\theta;A_T^*) \le \H(A_T^*) \le \log T$.  For each $t=0,\dots,T-1$, the agent's action $A_t$ also lies in $\{1,\dots,T\}$, so
    \[B_{\E_\pi[R_{t+1}]} \le B_{\E[\rewardMean{\theta}{A_T^*}]} \le \log T.\]
    Averaging over $t=0,1,\ldots,T-1$ gives $\bitEquivAvg{T}{\pi} \le \log T$.  Hence, the environment is non-open-ended.  
\end{proof}

\section{Proofs for \cref{sec:linear_gaussian_bandit}}
\label{app:proofs_for_openended}

\subsection{Proof of \cref{thm:TS_failure}}
\TSFailure*
\begin{proof}
	For any valid history $H_t$, the posterior $\P(\theta\in\cdot \mid H_t)$ only changes from the prior for finitely many coordinates.  Let the set $J_t = \{j \in \mathbb{N}: A_{\tau, j} = 1 \text{ for some } \tau \le t \}$ denote the set of such coordinates.  Then $|J_t| < \infty$ almost surely, and for all $j \notin J_t$, we have $\theta^{(t)}_j \sim \mathcal{N}(-1,1)$.  
	
	For a TS agent, if a sampled $\theta^{(t)}_j > 0$, then the action taken should set $a_j = 1$.  However, since there are infinitely many $j\notin J_t$, infinitely many coordinates of $a$ should be set to $1$.  However, such an action does not reside in $\actions = \{a_i \in \{0,1\}^\mathbb{N} \mid \|a\|_1 < \infty\}$, and hence is not valid. 
\end{proof}

\subsection{Proof of \cref{thm:TS_failure_reward}}
\TSFailureReward*
\begin{proof}
	For any valid history $H_t$, the posterior $\P(\theta\in\cdot \mid H_t)$ only changes from the prior for finitely many coordinates.  Let the set $J_t = \{j \in \mathbb{N}: A_{\tau, j} = 1 \text{ for some } \tau \le t \}$ denote the set of such coordinates.  Then $|J_t| < \infty$ almost surely, and for all $j \notin J_t$, we have $\theta^{(t)}_j \sim \mathcal{N}(-1,1)$.  
	
	In the truncated problem, TS sets $A_{t,M,j}=1$ whenever $\theta^{(t)}_j>0$.  Let $p = \P(Z > 0)$ for $Z\sim\mathcal{N}(-1,1)$.  Then $p > 0$ and for all $j\notin J_t$, 
	\[\E[A_{t,M,j}\theta_j\mid H_t] = \P(\theta^{(t)}_j>0\mid H_t)\E[\theta_j\mid H_t]= -p.\]
	Therefore,
	\[\E[A_{t,M}^\top \theta \mid H_t] = \sum_{j\in J_t\cap [M]} \E[A_{t,M,j}\theta_j\mid H_t] - p |[M]\setminus J_t|.\]
	The first term is eventually constant in $M$, because $J_t$ is finite, whereas $|[M]\setminus J_t|\to\infty$. Since this holds for every valid history, the claim follows.
\end{proof}

\subsection{Proof of \cref{thm:FT_failure}}
\FTFailure*
\begin{proof}
	A fixed-$M$ FT agent only selects actions in $\actions_M$, and $|\actions_M|=2^M$.  For each $t$,
	\[\bitEquiv{\E_{\pi^M}[R_{t+1}]} \le \I(\theta;A_t).\]
	Since $A_t$ takes values in a set of size $2^M$, 
	\[\I(\theta;A_t) \le \H(A_t) \le \log |\actions_M| = M\log 2.\]
	Averaging over $t=0,\ldots,T-1$ gives 
	\[\bitEquivAvg{T}{\pi^M} \le M\log 2.\]
	Hence no fixed-$M$ FT agent achieves open-ended learning.
\end{proof}

\subsection{Proof of \cref{thm:TTS_success}}

The argument has two parts.
First, we show that, in the \openEndedBandit{}, attaining large expected reward
requires a proportional amount of information about the environment.  
This establishes that linear reward growth implies linear growth of the average bit-equivalent. 
Second, we show that the cumulative reward of the TTS agent grows quadratically in time.
These show that TTS attains open-ended learning.

\subsubsection{Reward growth implies bit-equivalent growth}

We begin with two information-theoretic lemmas.  The first is a standard
Gaussian KL lower bound: among distributions with a prescribed mean, a translated
standard Gaussian minimizes the KL divergence to the standard Gaussian.

\begin{lemma}
\label{lem:kl_lower_bound}
    Let $P = \mathcal{N}(0,I_d)$ and let $Q$ be any probability distribution on $\R^d$ with mean $m \in \R^d$.  Then 
    \[\KL(Q\| P) \ge \frac{1}{2}\|m\|_2^2,\]
    with unique minimizer $Q = \mathcal{N}(m, I_d)$. 
\end{lemma}
\begin{proof}
    Let $P_m = \mathcal{N}(m,I_d)$.  Then
    \begin{align*}
        \log\frac{dP_m}{dP}(x) &= \log p_m(x) - \log p(x) = \frac{1}{2}\|x\|_2^2 - \frac{1}{2} \|x-m\|_2^2 = x^\top m - \frac{1}{2} \|m\|_2^2.
    \end{align*}
    Since $\log\frac{dQ}{dP} = \log\frac{dQ}{dP_m} + \log\frac{dP_m}{dP}$, taking expectation under $Q$ gives
    \begin{align*}
        \KL(Q\|P) &= \KL(Q\|P_m) +\E_Q\left[ X^\top m - \frac{1}{2}\|m\|_2^2\right] = \KL(Q\|P_m) + \frac{1}{2} \|m\|_2^2,
    \end{align*}
    because $\E_Q[X]=m$.  The claim follows from the nonnegativity of KL divergence, with equality iff $Q = P_m$.  
\end{proof}

The next lemma extends this finite-dimensional inequality to an
infinite-dimensional standard Gaussian sequence.  It shows that finite mutual
information controls the squared norm of the posterior mean.

\begin{lemma}
\label{lem:gaussian_info}
    Let $X = (X_i)_{i\ge 1}$ with $X_i\stackrel{iid}{\sim}\mathcal{N}(0,1)$, and let $Y$ be any random variable with $\I(X;Y) < \infty$.  Then
    \[\E \left[ \| \E[X|Y] \|_2^2\right] \le 2\cdot \I(X;Y).\]
\end{lemma}
\begin{proof}
    For $d\ge 1$, let $X^{(d)} = (X_1,\ldots,X_d)$.  Let $m(y) = \E[X^{(d)}|Y=y]$.  By the finite-dimensional argument from \cref{lem:kl_lower_bound}, 
    \[\KL(P_{X^{(d)}|Y=y} \| X^{(d)}) \ge \frac{1}{2} \|m(y)\|^2_2\] in nats.  Taking the expectation over $Y$ gives
    \[\I(X^{(d)};Y) \ge \frac{1}{2}\E\left[ \|m(Y)\|_2^2 \right] \implies 2\I(X^{(d)};Y) \ge \E\left[ \|\E[X^{(d)}|Y=y]\|_2^2 \right]\]
    in bits.  Since $X^{(d)}$ is a measurable function of $X$, the data-processing inequality gives 
    \[ \I(X;Y)\ge \I(X^{(d)};Y).\]
    Therefore, 
    \[2\I(X;Y)\ge \E\left[\sum_{i=1}^d \E\left[X_i\mid Y\right]^2 \right]\]
    for all $d$.  Taking $d\to\infty$ and applying monotone convergence theorem yields
    \[2\I(X;Y)\ge \E\left[\sum_{i=1}^\infty \E\left[X_i\mid Y\right]^2 \right].\]
\end{proof}

We now prove that reward is costly in information.  In the \openEndedBandit{}, every positive unit of expected
reward requires at least two nats of information about the parameter.

\begin{lemma}
\label{thm:linear_reward_growth_implies_linear_info_gain}
In the \openEndedBandit{}, for all $\rho \in \R$, $B_{\rho} \ge 2\rho$. 
\end{lemma}
\begin{proof}
    The case where $\rho<0$ is trivial; so we focus on the case where $\rho \ge 0$.
    For any policy $\pi$, let $A\sim \pi$.  The constraint $\E\left[ \rewardMean{\theta}{A}\right] \ge \rho$ is equivalent to 
    \[\E[ A^\top \theta]  = \E[\langle A,\theta\rangle] \ge \rho.\]

    Write $\theta = -\mathbf{1} + X$ where $X_i = \theta_i + 1$ are i.i.d. $\mathcal{N}(0, 1)$.  We have $\I(\theta;A) = \I(X;A)$ since mutual information is invariant under translations.  Then
    \[\E[\langle A,\theta\rangle] = - \E\left[ \|A\|_1\right] + \E\left[ \langle A,X\rangle\right].\]

    By Cauchy-Schwarz inequality, 
    \[\E \left[ \langle A, X\rangle\right] = \E \left[ \langle A, \E[X|A]\rangle\right] \le \sqrt{\E \left[ \|A\|_2^2\right]} \sqrt{\E \left[ \| \E[X|A] \|_2^2 \right]} \le \sqrt{\E \left[ \|A\|_2^2\right]}  \sqrt{2\I(X;A)}.\]
    where the last inequality follows by applying \cref{lem:gaussian_info}.
    
    Since $A_i \in  \{0,1\}$, $\|A\|_2^2 = \sum_i A_i^2 \le \sum_i A_i = \|A\|_1$.  Let $x = \E \left[ \|A\|_2^2\right]$, then
    \[\E[\langle A,\theta\rangle] \le - x + \sqrt{x}  \sqrt{2 \I(X;A)}.\]
    Maximizing the RHS over $x \ge 0$ yields  
    \[\E[\langle A,\theta\rangle] \le \frac{1}{2}\I(X;A).\]
    Therefore, if $\E[\langle A,\theta\rangle] \ge \rho\ge 0$, then 
    \[\I(\theta;A) = \I(X;A) \ge 2\rho.\]
    Taking the infimum over all $\pi$ satisfying $\E_\pi[R_{t+1}] \ge \rho$ concludes the proof.
\end{proof}

\begin{corollary}
\label{coro:linear_reward_implies_linear_info}
In the \openEndedBandit{}, let $\pi$ be an agent and let $T \ge 1$. Then
\begin{align*}
    \bitEquivAvg{T}{\pi}
    \ge
    \frac{2}{T}
    \E_\pi\left[
        \sum_{t=0}^{T-1} R_{t+1}
    \right].
\end{align*}
In particular, if
\begin{align*}
    \E_\pi\left[
        \sum_{t=0}^{T-1} R_{t+1}
    \right]
    =
    \Omega(T^2),
\end{align*}
then
\begin{align*}
    \bitEquivAvg{T}{\pi}
    =
    \Omega(T).
\end{align*}
\end{corollary}

\begin{proof}
By \cref{thm:linear_reward_growth_implies_linear_info_gain}, for each
$t=0,\ldots,T-1$,
\begin{align*}
    \bitEquiv{\E_\pi[R_{t+1}]}
    \ge
    2\E_\pi[R_{t+1}].
\end{align*}
Averaging over $t$ gives
\begin{align*}
    \bitEquivAvg{T}{\pi}
    &=
    \frac{1}{T}
    \sum_{t=0}^{T-1}
    \bitEquiv{\E_\pi[R_{t+1}]}
    \\
    &\ge
    \frac{2}{T}
    \sum_{t=0}^{T-1}
    \E_\pi[R_{t+1}]
    \\
    &=
    \frac{2}{T}
    \E_\pi\left[
        \sum_{t=0}^{T-1} R_{t+1}
    \right].
\end{align*}
The final claims follow immediately.
\end{proof}

\subsubsection{TTS attains linear average bit-equivalent}

We will use the following finite-dimensional Thompson-sampling regret bound
from \citet{zhu2026prior}.  Although the result is stated for centered
Gaussian priors, the proof extends without change to Gaussian priors with
arbitrary means, since translating the prior mean only shifts the posterior
mean and does not affect the covariance-based terms in the regret analysis.

\begin{theorem}[Finite-dimensional linear-Gaussian TS regret; Theorem~1 of \citet{zhu2026prior}]
\label{thm:fixed_dim_ts_regret}
Consider a $d$-dimensional linear-Gaussian bandit with action set
$\mathcal A\subseteq \{a\in\mathbb R^d:\|a\|_2\le r\}$, prior
\[
    \theta\sim\mathcal N(\mu_0,\Sigma_0),
\]
and observations
\[
    R_{t+1}=A_t^\top\theta+W_{t+1},
    \qquad
    W_{t+1}\sim\mathcal N(0,\sigma^2).
\]
Let $\pi_{\mathrm{TS}}$ denote Thompson sampling.  Define
\[
    C_1(d,T)
    =
    \sqrt{
        1+
        \max\left\{
            24\log(T/d),
            \sqrt{24\log(T/d)}
        \right\}
    },
\]
and
\[
    C_2(d,T,\sigma,r,\Sigma_0)
    =
    C_1(d,T)
    \sqrt{
        2\log\left(
            1+\frac{r^2\|\Sigma_0\|_{\op}T}{d\sigma^2}
        \right)
    }.
\]
Then, for all horizons $T\ge d$,
\[
    \Reg_{\mathrm{TS}}(T)
    :=
    \E\left[
        \sum_{t=0}^{T-1}
        \left(
            \max_{a\in\mathcal A}a^\top\theta
            -
            A_t^\top\theta
        \right)
    \right]
\]
satisfies
\[
    \Reg_{\mathrm{TS}}(T)
    \le
    d\sigma\sqrt T\,
    C_2(d,T,\sigma,r,\Sigma_0)
    +
    3r\sqrt d\,
    \Tr(\Sigma_0^{1/2})C_1(d,T)
    +
    \sqrt{2r^2\Tr(\Sigma_0)}.
\]
\end{theorem}

\TTSSuccess*
\begin{proof}
By \cref{coro:linear_reward_implies_linear_info}, it suffices to construct a
truncation schedule for which
\begin{align}
    \E_{\pi_{\mathrm{TTS}}}
    \left[
        \sum_{t=0}^{T-1} R_{t+1}
    \right]
    =
    \Omega(T^2).
    \label{eq:tts-suffices-quadratic-reward}
\end{align}

For \(m\ge 1\), let
\begin{align*}
    \actions_m
    =
    \left\{
        a\in\actions : a_i=0 \text{ for all } i>m
    \right\}.
\end{align*}
For the \(m\)-coordinate truncation, the optimal action is
\begin{align*}
    a_i^*(\theta)
    =
    \one\{\theta_i>0\},
    \qquad i=1,\ldots,m.
\end{align*}
Thus
\begin{align}
    \max_{a\in\actions_m}\theta^\top a
    =
    \sum_{i=1}^m(\theta_i)_+ .
    \label{eq:tts-truncated-optimal-value}
\end{align}
Since \(\theta_i\sim\mathcal N(-1,1)\), define
\begin{align}
    \eta
    :=
    \E[(\theta_i)_+]
    =
    \phi(1)-\Phi(-1)
    >
    0,
    \label{eq:tts-eta-definition}
\end{align}
where \(\phi\) and \(\Phi\) denote the standard normal density and distribution
function. Hence
\begin{align}
    r_m
    :=
    \E\left[\max_{a\in\actions_m}\theta^\top a\right]
    =
    \eta m .
    \label{eq:tts-rm-linear}
\end{align}

We define epochs by
\begin{align*}
    S_k = 2^k-1,
    \qquad
    \mathcal E_k
    =
    \{S_k,S_k+1,\ldots,S_{k+1}-1\},
    \qquad
    n_k = |\mathcal E_k| = 2^k .
\end{align*}
Thus the epochs partition \(\{0,1,2,\ldots\}\).

Choose a constant \(\alpha\in(0,1/2)\). To make the choice precise, define
\begin{align}
    C_\alpha
    :=
    \sqrt{
        1+
        \max\left\{
            24\log\frac{2}{\alpha},
            \sqrt{24\log\frac{2}{\alpha}}
        \right\}
    } .
    \label{eq:tts-c-alpha-definition}
\end{align}
Since \(\alpha C_\alpha\to 0\) as \(\alpha\downarrow 0\), we may choose
\(\alpha\in(0,1/2)\) small enough that
\begin{align}
    3 C_\alpha \alpha
    <
    \frac{\eta}{32}.
    \label{eq:tts-alpha-choice}
\end{align}
For \(t\in\mathcal E_k\), set
\begin{align}
    m_t=m_k:=\max\{1,\lfloor \alpha n_k\rfloor\}.
    \label{eq:tts-mk-schedule}
\end{align}
Thus TTS uses a fixed \(m_k\)-coordinate truncation throughout epoch \(k\).

Fix an epoch \(k\). Condition on the history \(H_{S_k}\) at the beginning of
the epoch. During epoch \(k\), TTS is exactly Thompson sampling for the
finite-dimensional linear-Gaussian bandit with action set \(\actions_{m_k}\)
and prior equal to the current posterior marginal of \(\theta_{1:m_k}\). Let
\(\Sigma_k\) be this posterior covariance matrix. Since the initial covariance
is the identity and Gaussian conditioning only decreases covariance in the
positive-semidefinite order,
\begin{align}
    0\preceq \Sigma_k \preceq I_{m_k}.
    \label{eq:tts-posterior-covariance-bound}
\end{align}
Moreover, every action \(a\in\actions_{m_k}\) satisfies
\begin{align}
    \|a\|_2\le \sqrt{m_k}.
    \label{eq:tts-action-radius}
\end{align}

Apply \cref{thm:fixed_dim_ts_regret} conditionally on \(H_{S_k}\), with
dimension \(d=m_k\), radius \(r=\sqrt{m_k}\), covariance \(\Sigma_k\), and
horizon \(n_k\). By \eqref{eq:tts-posterior-covariance-bound},
\begin{align}
    \Tr(\Sigma_k)\le m_k,
    \qquad
    \Tr(\Sigma_k^{1/2})\le m_k,
    \qquad
    \|\Sigma_k\|_{\op}\le 1.
    \label{eq:tts-covariance-trace-bounds}
\end{align}
For all sufficiently large \(k\), \eqref{eq:tts-mk-schedule} gives
\begin{align}
    \frac{\alpha}{2}n_k
    \le
    m_k
    \le
    \alpha n_k .
    \label{eq:tts-mk-linear-bounds}
\end{align}
Thus \(n_k/m_k\le 2/\alpha\), and the \(C_1\)-constant in
\cref{thm:fixed_dim_ts_regret} is bounded by \(C_\alpha\). Also, using
\eqref{eq:tts-action-radius} and \eqref{eq:tts-covariance-trace-bounds},
\begin{align}
    C_2(m_k,n_k,\sigma,\sqrt{m_k},\Sigma_k)
    \le
    C_\alpha
    \sqrt{
        2\log\left(1+\frac{n_k}{\sigma^2}\right)
    } .
    \label{eq:tts-c2-bound}
\end{align}
Therefore the conditional Bayesian regret over epoch \(k\) is bounded by a
deterministic quantity \(b_k\) satisfying
\begin{align}
    b_k
    \le
    \sigma m_k\sqrt{n_k}\,
    C_\alpha
    \sqrt{
        2\log\left(1+\frac{n_k}{\sigma^2}\right)
    }
    +
    3 C_\alpha m_k^2
    +
    \sqrt{2}\,m_k .
    \label{eq:tts-bk-bound}
\end{align}
Combining \eqref{eq:tts-bk-bound} with \eqref{eq:tts-mk-linear-bounds} yields
\begin{align}
    \frac{b_k}{n_k^2}
    \le
    3C_\alpha\alpha^2
    +
    o(1).
    \label{eq:tts-bk-asymptotic}
\end{align}
By \eqref{eq:tts-alpha-choice}, there exists \(k_0\) such that for all
\(k\ge k_0\),
\begin{align}
    b_k
    \le
    \frac{\eta\alpha}{16}n_k^2 .
    \label{eq:tts-bk-small}
\end{align}

Let
\begin{align*}
    V_k(\theta)
    :=
    \max_{a\in\actions_{m_k}}\theta^\top a .
\end{align*}
For \(0\le s\le n_k\), define the expected reward in the first \(s\) rounds of
epoch \(k\) by
\begin{align}
    G_k(s)
    :=
    \E_{\pi_{\mathrm{TTS}}}
    \left[
        \sum_{\tau=0}^{s-1}
        R_{S_k+\tau+1}
    \right].
    \label{eq:tts-gks-definition}
\end{align}
The instantaneous regret \(V_k(\theta)-\theta^\top A_t\) is nonnegative during
epoch \(k\), because \(A_t\in\actions_{m_k}\). Therefore the regret over any
prefix of the epoch is at most the regret over the full epoch. Hence
\begin{align}
    G_k(s)
    \ge
    s\,\E[V_k(\theta)] - b_k
    =
    s\,r_{m_k} - b_k .
    \label{eq:tts-prefix-reward-lower-bound}
\end{align}
For a full epoch, \eqref{eq:tts-rm-linear},
\eqref{eq:tts-mk-linear-bounds}, and \eqref{eq:tts-bk-small} imply
\begin{align}
    G_k(n_k)
    &\ge
    \eta m_k n_k - b_k
    \nonumber\\
    &\ge
    \frac{\eta\alpha}{2}n_k^2
    -
    \frac{\eta\alpha}{16}n_k^2
    \nonumber\\
    &=
    \frac{7\eta\alpha}{16}n_k^2 .
    \label{eq:tts-full-epoch-reward-lower-bound}
\end{align}
For an arbitrary prefix, \eqref{eq:tts-prefix-reward-lower-bound} and
\eqref{eq:tts-bk-small} give
\begin{align}
    G_k(s)
    \ge
    -b_k
    \ge
    -\frac{\eta\alpha}{16}n_k^2 .
    \label{eq:tts-prefix-reward-negative-bound}
\end{align}

Now let \(T\ge 1\), and let \(K\) be the epoch containing time \(T-1\). Then
\begin{align}
    S_K\le T-1<S_{K+1},
    \qquad
    n_K=\Theta(T).
    \label{eq:tts-final-epoch-size}
\end{align}
The cumulative reward up to time \(T\) is the sum of rewards from completed
epochs plus a prefix of epoch \(K\). Ignoring the finite number of initial
epochs before \(k_0\), and applying
\eqref{eq:tts-full-epoch-reward-lower-bound} and
\eqref{eq:tts-prefix-reward-negative-bound}, we get
\begin{align}
    \E_{\pi_{\mathrm{TTS}}}
    \left[
        \sum_{t=0}^{T-1}R_{t+1}
    \right]
    &\ge
    \sum_{k=k_0}^{K-1}
    \frac{7\eta\alpha}{16}n_k^2
    -
    \frac{\eta\alpha}{16}n_K^2
    -
    O(1).
    \label{eq:tts-total-reward-before-geometric-sum}
\end{align}
Since \(n_k=2^k\),
\begin{align}
    \sum_{k=k_0}^{K-1}n_k^2
    =
    \sum_{k=k_0}^{K-1}4^k
    =
    \frac{4^K-4^{k_0}}{3}
    =
    \frac{n_K^2}{3}-O(1).
    \label{eq:tts-geometric-sum}
\end{align}
Combining \eqref{eq:tts-total-reward-before-geometric-sum} and
\eqref{eq:tts-geometric-sum} yields
\begin{align}
    \E_{\pi_{\mathrm{TTS}}}
    \left[
        \sum_{t=0}^{T-1}R_{t+1}
    \right]
    &\ge
    \left(
        \frac{7\eta\alpha}{48}
        -
        \frac{\eta\alpha}{16}
    \right)n_K^2
    -
    O(1)
    \nonumber\\
    &=
    \frac{\eta\alpha}{12}n_K^2
    -
    O(1).
    \label{eq:tts-total-reward-nk-square}
\end{align}
By \eqref{eq:tts-final-epoch-size},
\eqref{eq:tts-total-reward-nk-square} implies
\begin{align}
    \E_{\pi_{\mathrm{TTS}}}
    \left[
        \sum_{t=0}^{T-1}R_{t+1}
    \right]
    =
    \Omega(T^2).
    \label{eq:tts-quadratic-cumulative-reward}
\end{align}
Finally, applying \cref{coro:linear_reward_implies_linear_info} and
\eqref{eq:tts-quadratic-cumulative-reward} gives
\begin{align*}
    \bitEquivAvg{T}{\pi_{\mathrm{TTS}}}
    \ge
    \frac{2}{T}
    \E_{\pi_{\mathrm{TTS}}}
    \left[
        \sum_{t=0}^{T-1}R_{t+1}
    \right]
    =
    \Omega(T).
\end{align*}
Thus TTS achieves open-ended learning.
\end{proof}

\subsection{The logistic bandit}

As discussed in \cref{sec:linear_gaussian_bandit}, a variant of the \openEndedBandit{} -- the \openEndedLogisticBandit{} -- remains open-ended even though rewards are bounded.  We give a formal proof when rewards are deterministic for simplicity and readability of the analysis.  Extending to the usual logistic bandit setting, i.e., reward of $A_t$ is a Bernoulli random variable with mean $\rewardMean{\theta}{A_t}$, does not pose additional technical difficulties.  
\begin{theorem}
\label{thm:generalized_linear_bandit_openended}
Suppose $\actions = \{a_i\in\{0,1\}^\mathbb{N}  \mid \|a\|_1<\infty \}$, $\theta\in\Re^\mathbb{N}$ is an infinite-dimensional random vector such that $\theta_i\overset{iid}{\sim}\mathcal{N}(-1,1)$, and 
\[\rewardMean{\theta}{A_t} = g(\langle \theta, A_t\rangle),\] 
where $g(x) = (1 + e^{-x})^{-1}$ is a logistic function.  Suppose there is no observation noise.  Then this environment is open-ended.  
\end{theorem}
The following lemma provides a useful lower bound on the bit-equivalent for this bandit instance.  
\begin{lemma}
\label{lem:logistic_bit_equiv_lower_bound}
	In the logistic bandit of \cref{thm:generalized_linear_bandit_openended}, for all $\rho \in [0,1)$, 
	\[B_{\rho} \ge 2\log\frac{1}{1-\rho} - 2\log 2.\]
\end{lemma}
\begin{proof}
	For any policy $\pi$, let $A\sim\pi$.  Suppose $\E[g(\langle A, \theta\rangle)] \ge \rho$. 

    Let \[f(A,\theta) = -2\log(1 - g(\langle A,\theta\rangle)) = 2\log(1 + e^{\langle A,\theta\rangle}),\]
    and let $P = P_{A,\theta}$ be the joint distribution of $A$ and $\theta$, and $Q = P_A P_\theta$ be the product of the marginals.  By the Donsker-Varadhan variational formula of KL-divergence \citep{Donsker1983Asymptotic}, 
    \begin{equation}
    \label{eqn:donsker-varadhan}
    \I(\theta;A) = \KL(P\|Q) \ge \E_P\left[ f(A,\theta) \right] - \log \E_Q \left[ \exp\left( f(A,\theta)\right)\right].
    \end{equation}
    Since $-2\log(1-x)$ is convex in $x$, by Jensen's inequality,
    \begin{align*}
        \E_P\left[ f(A,\theta) \right] 
        &\ge -2 \log \left( 1 - \E \left[ g(\langle A,\theta\rangle)\right]\right) \ge -2\log(1-\rho) = 2\log\frac{1}{1-\rho},
    \end{align*}
    where the last inequality follows since $\E[g(\langle A,\theta\rangle)] \ge \rho$.  
    It remains to bound the second term in \eqref{eqn:donsker-varadhan}.  When $A=a$ and $\|a\|_1 = n$ for a fixed $n\in\mathbb{N}$, we have $\langle \theta,a\rangle \sim \mathcal{N}(-n,n)$.  Conditioning on $A=a$ with $\|a\|_1 = n$,
    \begin{align*}
        \E_Q[\exp(f(A,\theta)) \mid A=a] &= \E_\theta[(1 + e^{\langle a,\theta \rangle})^2] \\
        &= 1 + 2\E_\theta[e^{\langle a,\theta \rangle}] + E_\theta[e^{2\langle a,\theta \rangle}]\\
        &= 1 + 2e^{-n/2} + 1 = 2 + 2e^{-\|a\|_1/2} \le 4.
    \end{align*}
    Putting these together, we have $\I(\theta;A) \ge 2\log\frac{1}{1-\rho} - 2\log 2$.  Taking the infimum over the distribution on $A$ satisfying $\E[g(\langle A,\theta \rangle)] \ge \rho$ proves the claim. 
\end{proof}
We now prove \cref{thm:generalized_linear_bandit_openended}.
\begin{proof}[Proof for \cref{thm:generalized_linear_bandit_openended}]
    Consider an agent $\pi$ that activates one coordinate at each step, and chooses the optimal value for each activated coordinate for all steps onward.  Since we assumed that the environment is deterministic, this agent will know the exact values of $\theta_i$ at time $i$.  We compute the expected reward attained by this agent. 

    Since $1 - g(x) = \frac{1}{1 + e^{x}} \le e^{-x}$ for all $x\in\R$, we have
    \begin{align*}
        1 - \E_\pi[R_{t+1}] &= \E\left[ 1 - g\left(\sum_{i=1}^{t} \theta_i^+ + \theta_{t+1}\right) \right] \\
        &\le \E\left[ e^{- \sum_{i=1}^{t} \theta_i^+ - \theta_{t+1}}\right] \\
        &= \E\left[ e^{- \sum_{i=1}^{t} \theta_i^+}\right] \E\left[ e^{-\theta_{t+1}}\right]\\
        &= \E[e^{-\theta_1^+}]^t \E\left[ e^{-\theta_{1}}\right]. 
    \end{align*}
    Since $\P(\theta_1 > 0) > 0$ and $e^{-\theta_1^+}<1$ on the event $\{\theta_1>0\}$, we have $\E[e^{-\theta_1^+}] < 1$.  Therefore, setting  $\alpha = -\log \E[e^{-\theta_1^+}] > 0$ and $C = \E\left[ e^{-\theta_{1}}\right]$, 
    \[1 - \E_\pi[R_{t+1}] \le C e^{-\alpha t}.\]
    By \cref{lem:logistic_bit_equiv_lower_bound}, \[\bitEquivAvg{T}{\pi} = \frac{1}{T}\sum_{t=0}^{T-1} B_{\E_\pi[R_{t+1}]} \ge \frac{1}{T}\sum_{t=0}^{T-1} B_{1 - C e^{-\alpha t}} \ge \frac{1}{T}\sum_{t=0}^{T-1}2\alpha t - 2\log C - 2\log 2 = \Omega(T).\]
    Hence, there exists an agent whose average bit-equivalent grows linearly in $T$, and so the environment is open-ended.  
\end{proof}

\subsection{Proof of \cref{thm:matching_upper_bound}}
We now turn to prove the matching upper bound in \cref{thm:matching_upper_bound}.  
To do this, we first prove an upper bound on the bit-equivalent in \openEndedBandit{}.  
\begin{lemma}[Linear upper bound on the bit-equivalent]
\label{lem:linear_upper_bound_on_bit_equiv}
In the \openEndedBandit{}, there exists a constant \(c_B<\infty\) such that
for every \(\rho\ge 0\),
\[
    B_\rho \le c_B(\rho+1).
\]
\end{lemma}

\begin{proof}
Let
\[
    p := \P(\theta_1>0)=1-\Phi(1),
    \qquad
    \eta := \E[\theta_1^+]
    =
    \E[(Z-1)^+]
    =
    \phi(1)-\{1-\Phi(1)\}>0,
\]
where \(Z\sim \mathcal N(0,1)\), and \(\phi,\Phi\) denote the standard normal
density and distribution function.

For \(m\in\mathbb N\), define the oracle action
\[
    A_i^{(m)}
    =
    \one\{i\le m,\ \theta_i>0\}.
\]
Then \(A^{(m)}\in\actions\) almost surely, and
\[
    \E[\langle \theta,A^{(m)}\rangle]
    =
    \sum_{i=1}^m \E[\theta_i^+]
    =
    m\eta.
\]
Moreover, \(A^{(m)}\) is determined by the signs of the first \(m\) coordinates.
The coordinates \(A_1^{(m)},\ldots,A_m^{(m)}\) are i.i.d. Bernoulli\((p)\),
so, writing \(h(p)=-p\log p-(1-p)\log(1-p)\),
\[
    \I(\theta;A^{(m)})
    \le
    \H(A^{(m)})
    =
    m h(p).
\]
Given \(\rho\ge 0\), choose \(m=\lceil \rho/\eta\rceil\). Then
\[
    \E[\langle \theta,A^{(m)}\rangle]\ge \rho,
\]
and therefore
\[
    B_\rho
    \le
    \I(\theta;A^{(m)})
    \le
    h(p)\left(\frac{\rho}{\eta}+1\right)
    \le
    c_B(\rho+1),
\]
for example with \(c_B=h(p)(\eta^{-1}+1)\).
\end{proof}

Then we show that the posterior means are bounded since the observations are scalar Gaussian.

\begin{lemma}[Posterior mean energy bound]
\label{lem:posterior_mean_energy_bound}
Let
\[
    X_i := \theta_i+1,
    \qquad i\ge 1,
\]
so that \(X_i\stackrel{\mathrm{iid}}{\sim}\mathcal N(0,1)\). For any policy
\(\pi\), any \(t\ge 0\), and
\[
    M_{t,i}:=\E_\pi[X_i\mid H_t],
\]
we have
\[
    \E_\pi\left[\sum_{i=1}^{\infty} M_{t,i}^2\right]\le t.
\]
\end{lemma}

\begin{proof}
Fix \(d\ge 1\). We first show that
\[
    \E_\pi\left[\sum_{i=1}^d M_{t,i}^2\right]\le t.
\]

Condition on a realized history \(H_t=h_t\). The realized actions
\(a_0,\ldots,a_{t-1}\) are then fixed finite-support vectors. Since
\(\theta=-\one+X\), the shifted observations satisfy
\[
    R_{s+1}+\|a_s\|_1
    =
    \langle a_s,X\rangle + W_{s+1},
    \qquad s=0,\ldots,t-1.
\]
Thus, conditional on \(H_t=h_t\), the posterior distribution of \(X\) is
Gaussian with covariance operator
\[
    C_t
    =
    \left(I+\frac{1}{\sigma^2}\sum_{s=0}^{t-1} a_s a_s^\top\right)^{-1}.
\]
Let \(P_d\) denote projection onto the first \(d\) coordinates. Then
\[
    \Cov(X_{1:d}\mid H_t=h_t)
    =
    P_d C_t P_d^\top .
\]
By the law of total variance,
\[
    \E_\pi\!\left[
        \left\|\E_\pi[X_{1:d}\mid H_t]\right\|_2^2
    \right]
    =
    d
    -
    \E_\pi\!\left[
        \Tr\!\left(\Cov(X_{1:d}\mid H_t)\right)
    \right].
\]
Therefore it suffices to upper bound the posterior trace reduction:
\[
    d-\Tr(P_d C_t P_d^\top).
\]

For each realized history, define
\[
    K_t := \frac{1}{\sigma^2}\sum_{s=0}^{t-1} a_s a_s^\top .
\]
Then \(K_t\succeq 0\) and \(\rank(K_t)\le t\). Since
\[
    C_t=(I+K_t)^{-1},
\]
we have
\[
    I-C_t
    =
    I-(I+K_t)^{-1}
    = K_t (I+K_t)^{-1}.
\]
The nonzero eigenvalues of \(I-C_t\) are
\[
    \frac{\lambda_j(K_t)}{1+\lambda_j(K_t)},
\]
where \(\lambda_j(K_t)\) ranges over the nonzero eigenvalues of \(K_t\).
Hence every eigenvalue of \(I-C_t\) lies in \([0,1]\), and
\[
    \rank(I-C_t)\le \rank(K_t)\le t.
\]
Therefore
\[
    \Tr(I-C_t)\le t.
\]
Since \(P_d^\top P_d\) is the projection onto the first \(d\) coordinates,
\[
    d-\Tr(P_d C_t P_d^\top)
    =
    \Tr\!\left(P_d(I-C_t)P_d^\top\right)
    \le
    \Tr(I-C_t)
    \le t.
\]
Thus, for every \(d\ge 1\),
\[
    \E_\pi\left[\sum_{i=1}^d M_{t,i}^2\right]\le t.
\]
Finally, by monotone convergence,
\[
    \E_\pi\left[\sum_{i=1}^{\infty}M_{t,i}^2\right]
    =
    \lim_{d\to\infty}
    \E_\pi\left[\sum_{i=1}^d M_{t,i}^2\right]
    \le t.
\]
\end{proof}

We are now ready to prove the upper bound.  
\MatchingUpperBound*
\begin{proof}
Let \(X_i=\theta_i+1\), and define
\[
    M_{t,i}:=\E_\pi[X_i\mid H_t].
\]
Then
\[
    \E_\pi[\theta_i\mid H_t]
    =
    -1+M_{t,i}.
\]

We first show that the expected one-step reward of any policy is at most linear
in \(t\). Since \(A_t\) is selected using \(H_t\) and possible policy
randomization independent of \(\theta\), we have
\[
    \E_\pi[\theta_i\mid H_t,A_t]
    =
    \E_\pi[\theta_i\mid H_t]
    =
    -1+M_{t,i}.
\]
Therefore,
\[
    \E_\pi[R_{t+1}]
    =
    \E_\pi[\langle \theta,A_t\rangle]
    =
    \E_\pi\left[
        \sum_{i\ge 1} A_{t,i}(-1+M_{t,i})
    \right].
\]
Since \(A_t\in\{0,1\}^{\mathbb N}\) has finite support almost surely,
\[
    \sum_{i\ge 1} A_{t,i}(-1+M_{t,i})
    \le
    \sum_{i\ge 1}(-1+M_{t,i})_+ .
\]
For every \(m\in\mathbb R\),
\[
    (m-1)_+ \le m^2.
\]
Hence
\[
    \E_\pi[R_{t+1}]
    \le
    \E_\pi\left[
        \sum_{i\ge 1} M_{t,i}^2
    \right].
\]
By \cref{lem:posterior_mean_energy_bound},
\[
    \E_\pi[R_{t+1}] \le t.
\]

The bit-equivalent \(B_\rho\) is nondecreasing in \(\rho\), because increasing
\(\rho\) only shrinks the feasible set in the definition of \(B_\rho\). Thus
\[
    B_{\E_\pi[R_{t+1}]}
    \le
    B_t.
\]
By \cref{lem:linear_upper_bound_on_bit_equiv},
\[
    B_t \le c_B(t+1).
\]
Therefore, for every \(T\ge 1\),
\[
    \bitEquivAvg{T}{\pi}
    =
    \frac1T\sum_{t=0}^{T-1}B_{\E_\pi[R_{t+1}]}
    \le
    \frac{c_B}{T}\sum_{t=0}^{T-1}(t+1)
    =
    \frac{c_B(T+1)}{2}.
\]
Hence
\[
    \bitEquivAvg{T}{\pi}=O(T),
\]
uniformly over policies \(\pi\).
\end{proof}

\end{document}